\documentclass{article}

\usepackage[accepted]{icml2016}
\usepackage{fancyhdr}
\usepackage{graphicx}
\usepackage{float}

\usepackage{amsmath}
\usepackage{amssymb}
\usepackage{amsthm}

\newtheorem{theorem}{Theorem}
\newtheorem{corollary}{Corollary}

\begin{document}

\twocolumn[
\icmltitle{DizzyRNN: Reparameterizing Recurrent Neural Networks for Norm-Preserving Backpropagation}
\icmlauthor{Victor Dorobantu*}{vdd6@cornell.edu}
\icmlauthor{Per Andre Stromhaug*}{pas282@cornell.edu}
\icmlauthor{Jess Renteria*}{jvr35@cornell.edu}
\icmladdress{Cornell University, Ithaca, NY}

\vskip 0.3in
]

\begin{abstract}

The vanishing and exploding gradient problems are well-studied obstacles that make it difficult for recurrent neural networks to learn long-term time dependencies. We propose a reparameterization of standard recurrent neural networks to update linear transformations in a provably norm-preserving way through Givens rotations. Additionally, we use the absolute value function as an element-wise non-linearity to preserve the norm of backpropagated signals over the entire network. We show that this reparameterization reduces the number of parameters and maintains the same algorithmic complexity as a standard recurrent neural network, while outperforming standard recurrent neural networks with orthogonal initializations and Long Short-Term Memory networks on the copy problem.

\end{abstract}

\section{Defining the problem}

Recurrent neural networks (RNNs) are trained by updating model parameters through gradient descent with backpropagation to minimize a loss function. However, RNNs in general will not prevent the loss derivative signal from decreasing in magnitude as it propagates through the network. This results in the \textit{vanishing gradient problem}, where the loss derivative signal becomes too small to update model parameters \cite{VanishingGradient}. This hampers training of RNNs, especially for learning long-term dependencies in data.

\section{Signal scaling analysis}

The prediction of an RNN is the result of a composition of linear transformations, element-wise non-linearities, and bias additions. To observe the sources of vanishing and exploding gradient problems in such a network, one can observe the minimum and maximum scaling properties of each transformation independently, and compose the resulting scaling factors.

\subsection{Linear transformations}

Let $y = Ax$ be an arbitrary linear transformation, where $A \in \mathbb{R}^{m \times n}$ is a matrix of rank $r$. 

\begin{theorem}
The singular value decomposition (SVD) of $A$ is $A = U \Sigma V^T$, for orthogonal $U$ and $V$, and diagonal $\Sigma$ with diagonal elements $\sigma_1, \dots, \sigma_n$, the singular values of $A$. 
\end{theorem}

From the SVD, Corollaries 1 and 2 follow.

\begin{corollary}

Let $\sigma_{min}$ and $\sigma_{max}$ be the minimum and maximum singular values of $A$, respectively. Then $\sigma_{min}\Vert x \Vert_2 \leq \Vert y \Vert_2 \leq \sigma_{max}\Vert x \Vert_2$.

\end{corollary}

\begin{corollary}

Let $\sigma_{min}$ and $\sigma_{max}$ be the minimum and maximum singular values of $A$, respectively. Then $\sigma_{min}$ and $\sigma_{max}$ are also the minimum and maximum singular values of $A^T$.

\end{corollary}

Proofs for these corollaries are deferred to the appendix.

Let $L$ be a scalar function of $y$. Then $$\frac{\partial L}{\partial x} = \frac{\partial L}{\partial y}\frac{\partial y}{\partial x} = A^T\frac{\partial L}{\partial y}$$ In an RNN, this relation describes the scaling effect of a linear transformation on the backpropagated signal. By Corollary 2, each linear transformation scales the loss derivative signal by at least the minimum singular value of the corresponding weight matrix and at most by the maximum singular value.

\begin{theorem}
All singular values of an orthogonal matrix are $1$.
\end{theorem}

By Corollary 2, if the linear transformation $A$ is orthogonal, then the linear transformation will not scale the loss derivative signal.

\subsection{Non-linear functions}

Let $y = f(x)$ be an arbitrary element-wise non-linear transformation. Let $L$ be a scalar function of $y$. Then $$\frac{\partial L}{\partial x} = \frac{\partial L}{\partial y}\frac{\partial y}{\partial x} = f'(x)\odot\frac{\partial L}{\partial y}$$ where $f'$ denotes the first derivative of $f$ and $\odot$ denotes the element-wise product. The $i$-th element of $\frac{\partial L}{\partial y}$ is scaled at least by $\min{\left(f'(x_i)\right)}$ and at most by $\max{\left(f'(x_i)\right)}$.

\subsection{Bias}

Let $y = x + b$ be an arbitrary addition of bias to $x$. Let $L$ be a scalar function of $y$. Then $$\frac{\partial L}{\partial x} = \frac{\partial L}{\partial y}\frac{\partial y}{\partial x} = \frac{\partial L}{\partial y}$$Though additive bias does not preserve the norm during a forward pass over the network, it does preserve the norm of the backpropagated signal during a backward pass.

\section{Previous Work}

In general, the singular values of weight matrices in RNNs are allowed to vary unbounded, leaving the network susceptible to the vanishing and exploding gradient problems. A popular approach to mitigating this problem is through orthogonal weight initialization, first proposed by Saxe et al. \cite{OrthogonalInit}. Later, identity matrix initialization was introduced for RNNs with ReLU non-linearities, and was shown to help networks learn longer time dependencies \cite{IRNN}. 

Arjovsky et al. \cite{uRNN} introduced the idea of an orthogonal reparametrization of weight matrices. Their approach involves composing several simple complex-valued unitary matrices, where each simple unitary matrix is parametrized  such that updates during gradient descent happen on the manifold of unitary matrices. The authors prove that their network cannot have an exploding gradient, and believe that this is the first time a non-linear network has  been proven to have this property.

Wisdom et al. \cite{wisdom} note that Arjovsky's approach does not parametrize all orthogonal matrices, and propose a method of computing the gradients of a weight matrix such that the update maintains orthogonality, but also allows the matrix to express the full set of orthogonal matrices.

Jia et al. \cite{jia} propose a method of regularizing the singular values during training by periodically computing the full SVD of the the weight matrices, and clipping the singular values to have some maximum allowed distance from $1$. The authors show this has comparable performance to batch normalization in convolutional neural networks. As computing the SVD is an expensive operation, this approach may not translate well to RNNs with large weight matrices.

\section{DizzyRNN}

We propose a simple method of updating orthogonal linear transformations in an RNN in a way that maintains orthogonality. We combine this approach with the use of the absolute value function as the non-linearity, thus constructing an RNN that provably has no vanishing or exploding gradient. We term an RNN using this approach a \textit{Dizzy Recurrent Neural Network} (DizzyRNN). The reparameterization maintains the same algorithmic space and time complexity as a standard RNN.

\subsection{Givens rotation}

An orthogonal matrix $A \in \mathbb{R}^{n\times n}$ may be constructed as a product of $n(n-1)/2$ Givens rotations \cite{GivensRotation}. Each rotation is a sparse matrix multiplication, depending on only two elements and modifying only two elements, meaning each rotation can be performed in $O(1)$ time. Additionally, each rotation is represented by one parameter: a rotation angle. These rotation angles can be updated directly using gradient descent through backpropagation.

Let $a$ and $b$ denote the indices of the fixed dimensions in one rotation, with $a < b$. Let $y = R_{a,b}(\theta)x$ express this rotation by an angle $\theta$. The rotation matrix $R_{a,b}(\theta)$ is sparse and orthogonal with the following form: each diagonal element is $1$ except for the $a$-th and $b$-th diagonal elements, which are $\cos{\theta}$. Additionally, two off-diagonal elements are non-zero; the element $(a, b)$ is $\sin{\theta}$ and the element $(b, a)$ is $-\sin{\theta}$. All remaining off-diagonal elements are $0$. Let $L$ be a scalar function of $y$. Then $$\frac{\partial L}{\partial x} = \frac{\partial L}{\partial y}\frac{\partial y}{\partial x} = R_{a,b}^T(\theta)\frac{\partial L}{\partial y}$$ Recall that since the matrix $R_{a,b}(\theta)$ is orthogonal with minimum and maximum singular values of $1$, the transpose $R_{a,b}^T(\theta)$ also has minimum and maximum singular values of $1$ (by Corollary 2). 

To update the rotation angles, note that the only elements of $y$ that differ from the corresponding element of $x$ are $y_a$ and $y_b$. Each can be expressed as $y_a = \cos{\theta}x_a + \sin{\theta}x_b$ and $y_b = -\sin{\theta}x_a + \cos{\theta}x_b$. The derivative of $L$ with respect to the parameter $\theta$ is thus
\begin{align*}
\frac{\partial L}{\partial \theta} &= \frac{\partial L}{\partial y_a}\frac{\partial y_a}{\partial \theta} + \frac{\partial L}{\partial y_b}\frac{\partial y_b}{\partial \theta}\\&=\begin{bmatrix}\frac{\partial L}{\partial y_a}&\frac{\partial L}{\partial y_b}\end{bmatrix}\begin{bmatrix}-\sin{\theta}&\cos{\theta}\\-\cos{\theta}&-\sin{\theta}\end{bmatrix}\begin{bmatrix}x_a\\x_b\end{bmatrix}
\end{align*}

To simplify this expression, define the matrix $E_{a,b}$ as $$E_{a,b} = \begin{bmatrix}e_a^T\\e_b^T\end{bmatrix}$$ where $e_i \in \mathbb{R}^n$ is a column vector of zeros with a $1$ in the $i$-th index. The matrix $E_{a,b}$ selects only the $a$-th and $b$-th indices of a vector. Additionally, define the matrix $R_{\partial}(\theta)$ as $$R_{\partial}(\theta) = \begin{bmatrix}-\sin{\theta}&\cos{\theta}\\-\cos{\theta}&-\sin{\theta}\end{bmatrix}$$ Note that $R_{\partial}(\theta)$ always has this form; it does not depend on indices $a$ and $b$. Now the derivative of $L$ with respect to the parameter $\theta$ can be represented as $$\frac{\partial L}{\partial \theta} = \left(E_{a,b}\frac{\partial L}{\partial y}\right)^T R_{\partial}(\theta)E_{a,b}x$$ This multiplication can be implemented in $O(1)$ time.

\subsection{Parallelization Through Packed Rotations}

While the DizzyRNNs maintain the same algorithmic complexity as standard RNNs, it is important to perform as many Givens rotations in parallel as possible in order to get good performance on GPU hardware. Since each Givens rotation only  affects two values in the input vector,  we can perform $n/2$ Givens rotations in parallel. We therefore only need $n-1$ sequential operations, each of which has $O(n)$ computational and space complexity. We refer to each of these $n-1$ operations as a \textit{packed rotation}, representable by a sparse matrix multiplication.

\subsection{Norm preserving non-linearity}

Typically used non-linearities like tanh and sigmoid strictly reduce the norm of a loss derivative signal during backpropagation. ReLU only preserves the norm in the case that each input element is non-negative. We propose the use of an element-wise absolute value non-linearity (denoted as \textit{abs}). Let $y = abs(x)$ be the element-wise absolute value of $x$, and let $L$ be a scalar function of $y$. Then $$\frac{\partial L}{\partial x} = \frac{\partial L}{\partial y}\frac{\partial y}{\partial x} = sign(x)\odot\frac{\partial L}{\partial y}$$ The use of this non-linearity preserves the norm of the backpropagated signal.

\subsection{Eliminating Vanishing and Exploding Gradients}
Let $P_1, \dots, P_{n-1}$ represent $n-1$ packed rotations, $h_t$ be a hidden state at time step $t$, $x_t$ be an input vector, and $b$ be a bias vector. Define the hidden state update equation as
$$ h_t = abs(P_1 \cdots P_{n-1} h_{t-1} + W_xx_t + b)$$

If $W_x$ is square, it can also be represented as $n-1$ packed rotations $Q_1, \dots, Q_n$, resulting in the hidden state update equation
$$ h_t = abs(P_1 \cdots P_{n-1} h_{t-1} + Q_1\cdots Q_{n-1}x_t + b)$$

\subsection{Eliminating Vanishing and Exploding Gradients}
Arvosky et al. claim to provide the first proof of a network having no exploding gradient (through their uRNN) \cite{uRNN}. We show that DizzyRNN has no exploding gradient and, more importantly, no vanishing gradient.

Let a state update equation for an RNN be defined as
$$h_t = f(W_h h_{t-1} + W_x x_t + b)$$ Let $L$ be a loss function over the RNN.

\begin{theorem}
In a DizzyRNN cell, $\left \Vert \frac{\partial L}{\partial h_t} \right \Vert_2 = \left \Vert \frac{\partial L}{\partial h_{t-1}} \right \Vert_2 $
\end{theorem}

\begin{theorem}
If $W_x$ is square, then $\left \Vert \frac{\partial L}{\partial x_t} \right \Vert_2 = \left \Vert \frac{\partial L}{\partial h_{t-1}} \right \Vert_2 $
\end{theorem}

Therefore, the network can propagate 
loss derivative signals through an arbitrarily large number of state updates and stacked cells. The proofs for Theorems $3$ and $4$ are deferred to the Appendix.

\section{Incorporating Singular Value Regularization} 
\subsection{Exposing singular values}

Let $y = Ax$ be an arbitrary linear transformation, where $A \in \mathbb{R}^{n \times n}$ is a matrix of rank $n$. Such a matrix can be represented by the DizzyRNN reparameterization through a construction $U\Sigma V^T$, where $U$ and $V$ are orthogonal matrix and $\Sigma$ is a diagonal matrix. This construction represents a singular value decomposition of a linear transformation; however, the diagonal elements of $\Sigma$ (the singular values) can be updated directly along with the rotation angles of $U$ and $V$. Additionally, the distribution of singular values can be penalized easily, regularizing the network while allowing full expressivity of linear transformations.

\subsection{Diagonal matrix}

A matrix-vector product $y = \Sigma x$ where $\Sigma \in \mathbb{R}^{n\times n}$ is a diagonal matrix can be represented as the element-wise vector product $y = \sigma \odot x$, where $\sigma$ is the vector of the diagonal elements of $\Sigma$. Let $L$ be a scalar function of $y$. Then \begin{align*}\frac{\partial L}{\partial x} &= \frac{\partial L}{\partial y}\frac{\partial y}{\partial x} = \sigma \odot \frac{\partial L}{\partial y}\\\frac{\partial L}{\partial \sigma} &= \frac{\partial L}{\partial y}\frac{\partial y}{\partial \sigma} = x \odot \frac{\partial L}{\partial y}\end{align*} Each computation can be performed in $O(n)$ time.

\subsection{Singular value regularization}

For a DizzyRNN, an additional term can be added to the loss function $L$ to penalize the distance of the singular values of each linear transformation from $1$. For each cell in the stack, let $\sigma$ denote the vector of all singular values of all linear transformations in the cell; the regularization term is then $\frac{1}{2}\lambda\Vert\sigma - e\Vert_2^2$, where $\lambda$ is a penalty factor and $e$ is the vector of all ones. The loss function can now be rewritten as $$L' = L + \frac{1}{2}\lambda\sum_{i=1}^{M}\Vert\sigma^{(i)}-e\Vert_2^2$$ for a DizzyRNN with a stack height of $M$ where $\sigma^{(i)}$ represents the vector of all singular values associated with the $i$-th cell in the stack. Note that setting the $\lambda$ hyperparameter to $0$ allows the singular values to grow or decay unbounded, and setting $\lambda$ to $\infty$ constrains each linear transformation to be orthogonal. Additionally, note that initializing the singular values of each linear transformation to $1$ is equivalent to an orthogonal initialization of the DizzyRNN.

\section{Experimental results}

We implemented DizzyRNN in Tensorflow and compared the performance of DizzyRNN with standard RNNs, Identity RNNs \cite{IRNN}, and Long Short-Term Memory networks (LSTM) \cite{LSTM}. We evaluated each network on the copy problem described in \cite{uRNN}. We modify the loss function in this problem to only quantify error on the copied portion of the output, making our baseline accuracy $10\%$ (guessing at random).

The copy problem for our experiments consisted of
memorizing a sequence of 10 one-hot vectors of length 10
and outputting the same sequence (via softmax) upon seeing a delimiter
after a time lag of 90 steps.

We use a stack size of $1$ and use only a subset of the total $n-1$
possible packed rotations for every orthogonal matrix.

All experiments consist of epochs with 10 batches of size 100, sampled directly from the underlying distribution.

\begin{figure}[h]
\includegraphics[width=0.5\textwidth, trim={0, 5cm, 0, 5cm}]{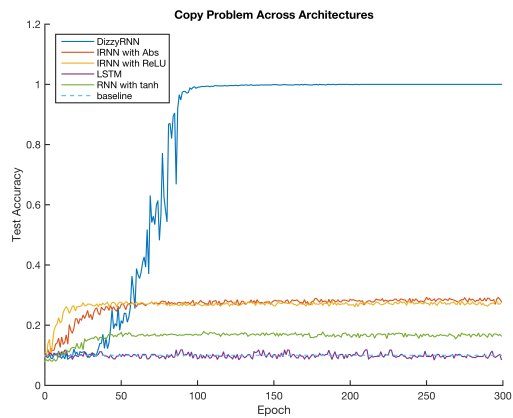}
\caption{Hidden state size of 128 across models. DizzyRNN has 10 packed rotations.}
\end{figure}

\begin{figure}[h]
\includegraphics[width=0.5\textwidth, trim={0, 5cm, 0, 5cm}]{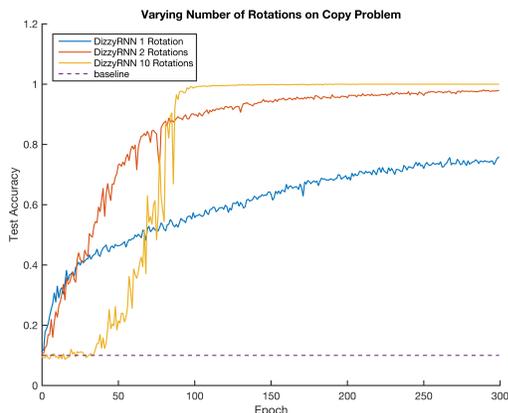}
\caption{Rotation here refers to a packed rotation. State size is fixed at 128.}
\end{figure}

DizzyRNN manages to reach near perfect accuracy in under 100 epochs
while other models either fail to break past the baseline or plateau
at a low test accuracy.
Note that 100 epochs corresponds to 100000 sampled training sequences.

\section{Conclusion}

DizzyRNNs prove to be a promising method of eliminating
the vanishing and exploding gradient problems.
The key is using pure rotations in combination with
norm-preserving non-linearities to force the norm
of the backpropagated gradient at each timestep to remain
fixed. Surprisingly, at least for the copy problem,
restricting weight matrices to pure rotations
actually improves model accuracy. This suggests that
gradient information is more valuable than model
expressiveness in this domain.

Further experimentation with sampling packed rotations will
be a topic of future work. Additionally, we would like to augment other state-of-the-art networks with Dizzy reparameterizations, such as Recurrent Highway Networks \cite{RHN}.

\bibliographystyle{icml2016}
\bibliography{submission}

\begin{thebibliography}{9}
\providecommand{\natexlab}[1]{#1}
\providecommand{\url}[1]{\texttt{#1}}
\expandafter\ifx\csname urlstyle\endcsname\relax
  \providecommand{\doi}[1]{doi: #1}\else
  \providecommand{\doi}{doi: \begingroup \urlstyle{rm}\Url}\fi

\bibitem[Arjovsky et~al.(2015)Arjovsky, Shah, and Bengio]{uRNN}
Arjovsky, Mart{\'{\i}}n, Shah, Amar, and Bengio, Yoshua.
\newblock Unitary evolution recurrent neural networks.
\newblock \emph{CoRR}, abs/1511.06464, 2015.
\newblock URL \url{http://arxiv.org/abs/1511.06464}.

\bibitem[Bengio et~al.(1994)Bengio, Simard, and Frasconi]{VanishingGradient}
Bengio, Yoshua, Simard, Patrice, and Frasconi, Paolo.
\newblock Learning long-term dependencies with gradient descent is difficult.
\newblock \emph{IEEE transactions on neural networks}, 5\penalty0 (2):\penalty0
  157--166, 1994.

\bibitem[Golub \& Van~Loan(2012)Golub and Van~Loan]{GivensRotation}
Golub, Gene~H and Van~Loan, Charles~F.
\newblock \emph{Matrix computations}, volume~3.
\newblock JHU Press, 2012.

\bibitem[Hochreiter \& Schmidhuber(1997)Hochreiter and Schmidhuber]{LSTM}
Hochreiter, Sepp and Schmidhuber, J{\"u}rgen.
\newblock Long short-term memory.
\newblock \emph{Neural computation}, 9\penalty0 (8):\penalty0 1735--1780, 1997.

\bibitem[Jia(2016)]{jia}
Jia, Kui.
\newblock Improving training of deep neural networks via singular value
  bounding.
\newblock \emph{arXiv preprint arXiv:1611.06013}, 2016.

\bibitem[Le et~al.(2015)Le, Jaitly, and Hinton]{IRNN}
Le, Quoc~V., Jaitly, Navdeep, and Hinton, Geoffrey~E.
\newblock A simple way to initialize recurrent networks of rectified linear
  units.
\newblock \emph{CoRR}, abs/1504.00941, 2015.
\newblock URL \url{http://arxiv.org/abs/1504.00941}.

\bibitem[Saxe et~al.(2013)Saxe, McClelland, and Ganguli]{OrthogonalInit}
Saxe, Andrew~M., McClelland, James~L., and Ganguli, Surya.
\newblock Exact solutions to the nonlinear dynamics of learning in deep linear
  neural networks.
\newblock \emph{CoRR}, abs/1312.6120, 2013.
\newblock URL \url{http://arxiv.org/abs/1312.6120}.

\bibitem[Wisdom et~al.(2016)Wisdom, Powers, Hershey, Le~Roux, and
  Atlas]{wisdom}
Wisdom, Scott, Powers, Thomas, Hershey, John, Le~Roux, Jonathan, and Atlas,
  Les.
\newblock Full-capacity unitary recurrent neural networks.
\newblock In \emph{Advances In Neural Information Processing Systems}, pp.\
  4880--4888, 2016.

\bibitem[Zilly et~al.(2016)Zilly, Srivastava, Koutn{\'\i}k, and
  Schmidhuber]{RHN}
Zilly, Julian~Georg, Srivastava, Rupesh~Kumar, Koutn{\'\i}k, Jan, and
  Schmidhuber, J{\"u}rgen.
\newblock Recurrent highway networks.
\newblock \emph{arXiv preprint arXiv:1607.03474}, 2016.

\end{thebibliography}

\section*{Appendix}

\setcounter{corollary}{0}
\setcounter{theorem}{0}

Let $y = Ax$ be an arbitrary linear transformation, where $A \in \mathbb{R}^{m \times n}$ is a matrix of rank $r$. 

\begin{theorem}

The singular value decomposition of $A$ is expressed as $A = U\Sigma V^T = \sum_{i=1}^r \sigma_iu_iv_i^T$, for orthogonal $U \in \mathbb{R}^{m \times r}$ and $V \in \mathbb{R}^{n \times r}$, and diagonal $\Sigma \in \mathbb{R}^{r \times r}$. $u_i$ and $v_i$ are the $i$-th columns of $U$ and $V$, respectively, and $\sigma_i$ is the $i$-th diagonal element of $\Sigma$. The columns of $U$ are the left singular vectors, the columns of $V$ are the right singular vectors, and the diagonal elements of $\Sigma$ are the singular values.

\end{theorem}

\begin{corollary}

Let $\sigma_{min}$ and $\sigma_{max}$ be the minimum and maximum singular values of $A$, respectively. Then $\sigma_{min}\Vert x \Vert_2 \leq \Vert y \Vert_2 \leq \sigma_{max}\Vert x \Vert_2$.

\end{corollary}

\begin{proof}

Express $y$ as $y = \sum_{i=1}^r \sigma_iu_iv_i^Tx$, i.e. a linear combination of the orthonormal set of left singular vectors. The $i$-th coefficient of the combination is the scalar $\sigma_iv_i^Tx$. Since the set of left singular vectors is orthonormal, the $l_2$-norm of $y$ is the Pythagorean sum of the coefficients of the linear combination, i.e., $\Vert y\Vert_2 = \left(\sum_{i=1}^r\left(\sigma_iv_i^Tx\right)^2\right)^{1/2}$. Note that the term $v_i^Tx$ is the magnitude of the projection of $x$ onto $v_i$. Without loss of generality, fix the norm of $x$ to $1$. Let $v_{min}$ and $v_{max}$ be the right singular vectors corresponding to $\sigma_{min}$ and $\sigma_{max}$, respectively. Then, the norm of $y$ is minimized for $x$ parallel to $v_{min}$, and maximized for $x$ parallel to $v_{max}$. The corresponding norms of $y$ are $\sigma_{min}v_{min}^Tv_{min}$ and $\sigma_{max}v_{max}^Tv_{max}$. Since the set of right singular vectors is orthonormal, $v_{min}^Tv_{min} = v_{max}^Tv_{max} = 1$, and the corresponding norms of $y$ are $\sigma_{min}$ and $\sigma_{max}$.

\end{proof}

\begin{corollary}

Let $\sigma_{min}$ and $\sigma_{max}$ be the minimum and maximum singular values of $A$, respectively. Then $\sigma_{min}$ and $\sigma_{max}$ are also the minimum and maximum singular values of $A^T$.

\end{corollary}

\begin{proof}

Since $A = U\Sigma V^T$ where $U$ and $V$ are orthogonal, $A^T = V\Sigma U^T$. By the same construction as in the previous corollary, if $y = A^Tx$, then $\Vert y\Vert_2 = \left(\sum_{i=1}^r\left(\sigma_iu_i^Tx\right)^2\right)^{1/2}$. For all $x$ such that $\Vert x\Vert_2 = 1$, the quantity is minimized and maximized for $x = u_{min}$ and $x = u_{max}$, respectively, where $u_{min}$ and $u_{max}$ are the left singular vectors corresponding to $\sigma_{min}$ and $\sigma_{max}$. The corresponding minimum and maximum norms are $\sigma_{min}$ and $\sigma_{max}$.

\end{proof}

Let a state update equation for an RNN be defined as
$$h_t = f(W_h h_{t-1} + W_x x_t + b)$$ Let $L$ be a loss function over the RNN.

\setcounter{theorem}{2}

\begin{theorem}
In a DizzyRNN cell, $\left \Vert \frac{\partial L}{\partial h_t} \right \Vert_2 = \left \Vert \frac{\partial L}{\partial h_{t-1}} \right \Vert_2 $
\end{theorem}

\begin{proof}

Let $y = W_h h_{t-1} + W_x x_t + b$ and express $$\frac{\partial L}{\partial h_{t-1}} = \frac{\partial L}{\partial h_t}\frac{\partial h_t}{\partial y} \frac{\partial y}{\partial h_{t-1}} = W_h^T \left (f'(y) \odot  \frac{\partial L}{\partial {h_t}}\right )$$ The $l_2$-norms of each side of this equation are $$  \left \Vert \frac{\partial L}{\partial h_{t-1}} \right \Vert_2= \left \Vert W_h^T \left (f'(y) \odot  \frac{\partial L}{\partial {h_t}}\right) \right \Vert_2$$ In a DizzyRNN, $f$ is the absolute value function, thus the elements of $f'(y)$ are $1$ or $-1$. $W_h$ is orthogonal since it is defined by a composition of Givens rotations. Neither $f'(y)$ and $W_h^T$ scale the norm of the vector $\frac{\partial L}{\partial h_t}$,  thus
$$ \left \Vert\frac{\partial L}{\partial h_{t-1}} \right \Vert_2 = \left \Vert \frac{\partial L}{\partial {h_t}} \right \Vert_2$$ 

\end{proof}

If instead $f$ is the ReLU non-linearity, then $\left \Vert \frac{\partial L}{\partial h_{t-1}} \right \Vert_2$ is only equal to $\left \Vert  \frac{\partial L}{\partial {h_t}} \right \Vert_2$ in the case where all values in $y$ are non-negative, resulting in a diminishing gradient in all other cases.

\begin{theorem}
If $W_x$ is square, then $\left \Vert \frac{\partial L}{\partial x_t} \right \Vert_2 = \left \Vert \frac{\partial L}{\partial h_{t-1}} \right \Vert_2 $
\end{theorem}

\begin{proof}
By symmetry with the proof of Theorem 3, the norms are shown to be equal.
\end{proof}

\begin{figure}[H]
\includegraphics[width=0.5\textwidth, trim={0, 5cm, 0, 5cm}]{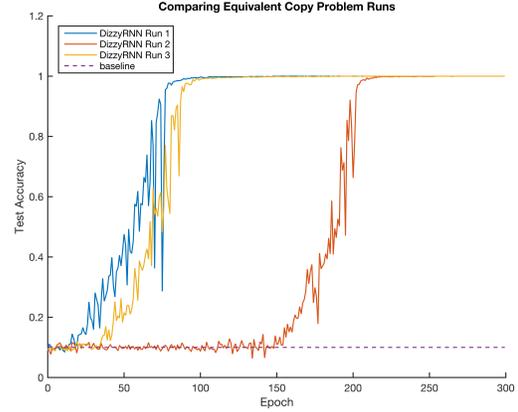}
\caption{Variance in convergence rates with constant hyperparameters. State size is 128 with 10 packed rotations.}
\end{figure}

\begin{figure}[H]
\includegraphics[width=0.5\textwidth, trim={0, 5cm, 0, 5cm}]{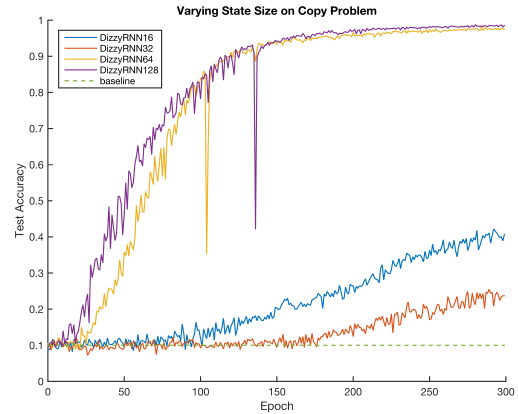}
\caption{The number of packed rotations is fixed at 5.
State sizes are 16, 32, 64, and 128}
\end{figure}

\begin{figure}[H]
\includegraphics[width=0.5\textwidth, trim={0, 5cm, 0, 5cm}]{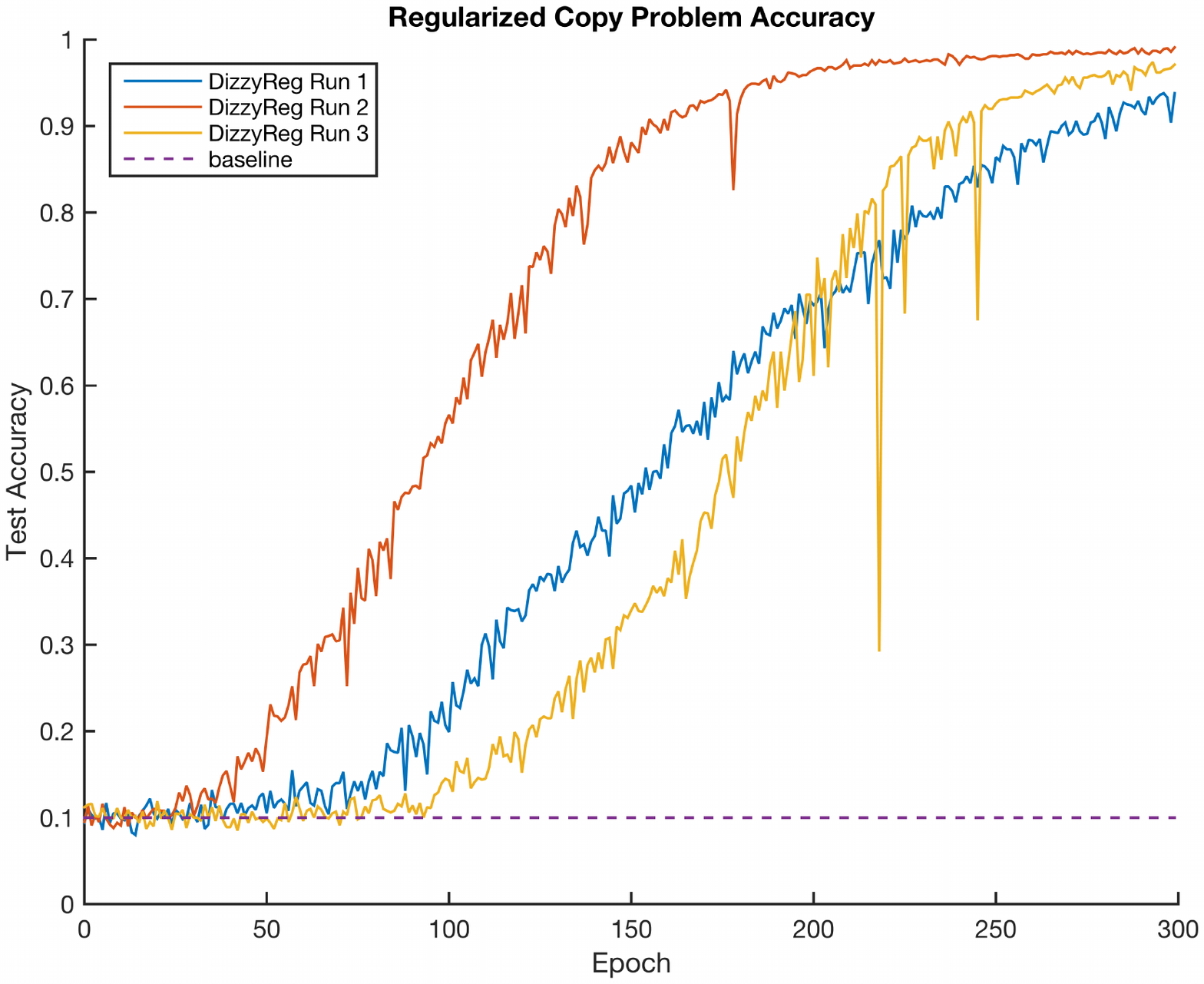}
\caption{Several regularized runs.}
\end{figure}

\begin{figure}[H]
\includegraphics[width=0.5\textwidth, trim={0, 5cm, 0, 5cm}]{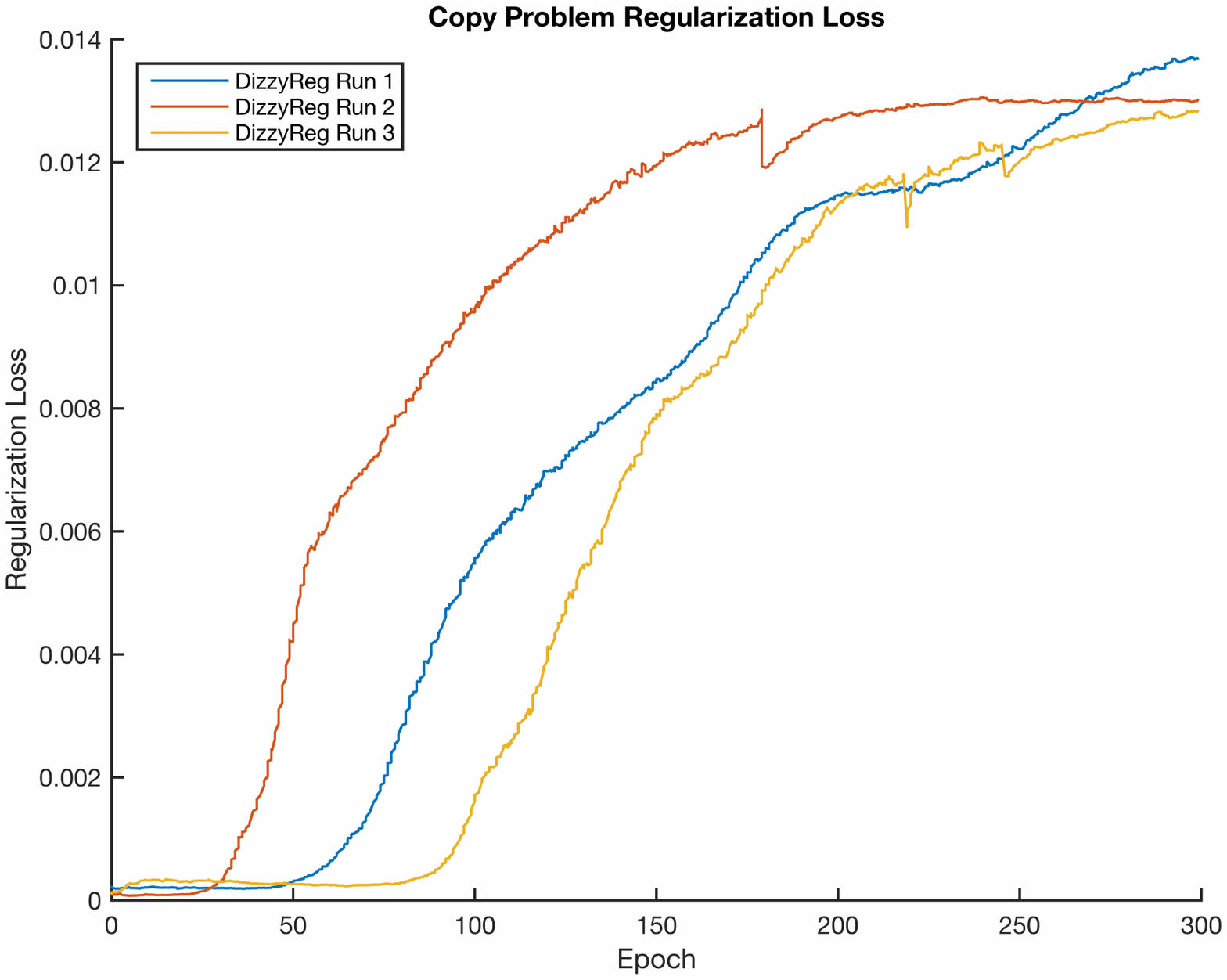}
\caption{Several regularized runs.}
\end{figure}

\end{document}